\documentclass[letterpaper, 10 pt, conference]{ieeeconf}  

\IEEEoverridecommandlockouts                              

\usepackage{amsmath} 
\usepackage{amssymb}  
\usepackage{mathtools}
\usepackage{bm}
\usepackage{lipsum}
\usepackage{graphicx}
\usepackage{algorithm}
\usepackage[noend]{algpseudocode}
\usepackage{comment}
\usepackage{multirow}

\newtheorem{define}{Definition}

\newtheorem{lemma}{Lemma}

\newtheorem{assume}{Assumption}

\newcommand{\R}{\mathbb R}

\newcommand{\Z}{\mathbb{Z}}

\newcommand{\bmx}[1]{\begin{bmatrix}#1\end{bmatrix}} 
\newcommand{\pth}[1]{\left(#1\right)} 

\DeclarePairedDelimiter{\ceil}{\lceil}{\rceil}
\DeclarePairedDelimiter{\floor}{\lfloor}{\rfloor}
\DeclarePairedDelimiter{\abs}{\lvert}{\rvert}

\makeatletter
\let\oldceil\ceil
\def\ceil{\@ifstar{\oldceil}{\oldceil*}}
\makeatother
\makeatletter
\let\oldfloor\floor
\def\floor{\@ifstar{\oldfloor}{\oldfloor*}}
\makeatother
\makeatletter
\let\oldnorm\norm
\def\norm{\@ifstar{\oldnorm}{\oldnorm*}}
\makeatother
\makeatletter
\let\oldabs\abs
\def\abs{\@ifstar{\oldabs}{\oldabs*}}
\makeatother

\title{\LARGE \bf
Triangle-Decomposable Graphs for Isoperimetric Robots
}

\author{Nathan Usevitch$^{1}$, Isaac Weaver$^{1}$, and James Usevitch$^{2}$
\thanks{*This work was partially supported by the Utah NASA Space Grant Consortium, (NASA Grant \#80NSSC20M0103), and through the 2024 NASA BIG Idea Challenge}
\thanks{$^{1}$Brigham Young University, Department of Mechanical Engineering}%
\thanks{$^{2}$Brigham Young University, Department of Electrical Engineering}
}

\begin{document}

\maketitle
\thispagestyle{empty}
\pagestyle{empty}

\begin{abstract}

Isoperimetric robots are large scale, untethered inflatable robots that can undergo large shape changes, but have only been demonstrated in one 3D shape- an octahedron. These robots consist of independent triangles that can change shape while maintaining their perimeter by moving the relative position of their joints. We introduce an optimization routine that determines if an arbitrary graph can be partitioned into unique triangles, and thus be constructed as an isoperimetric robotic system. We enumerate all minimally rigid graphs that can be constructed with unique triangles up to 9 nodes (7 triangles), and characterize the workspace of one node of each these robots. We also present a method for constructing larger graphs that can be partitioned by assembling subgraphs that are already partitioned into triangles. This enables a wide variety of isoperimetric robot configurations.

\end{abstract}

\section{Introduction}

Robotic systems will be capable of an increasing number of tasks if they can change shape to perform a variety of tasks and safely interact with people. One type of robot with the potential for large shape change and human-safe interaction is the isoperimetric robot, first introduced in \cite{usevitch2020untethered}, and with an example shown in Fig.~\ref{fig:isoperimetric}. This robot consists of inflated fabric beams as the primary structural members, with robotic rollers that pinch the tube, reducing the local bending stiffness to create a region of the tube that acts as a rotational joint. These rollers can drive up and down the tube, simultaneously lengthening one edge of the tube and shortening another. The overall internal volume inside the tubes is conserved, meaning that for inflated beams, no source of compressed air is needed. Computationally, the resulting structure approximates a computer mesh defined by edges and nodes, and allows the robot to change shape. The soft structure of the robot also gives it inherent compliance that makes the robot human safe. 

The only demonstrated 3D shape of an isoperimetric robot has been an octahedron, composed of 4 individual tubes that make up triangles \cite{usevitch2020untethered}.  The fundamental building block for these robots in 3 dimensions is an individual triangle. In a triangle, the axis of rotation formed by each of the roller modules are guaranteed to remain parallel. Robots can be formed by connecting different triangles together at spherical joints. The unique constraints that the robot be made of triangles creates a key question: what are the shapes of robots that can be formed from individual triangles such that each edge in the overall graph is the part of exactly one triangle? This paper presents techniques to answer this question and determine many new types of robots that can be created. We have constructed two of these robots as a demonstration in Fig.~\ref{fig:isoperimetric}.

\begin{figure}[tb]
\centering
\includegraphics[width=.85\columnwidth]{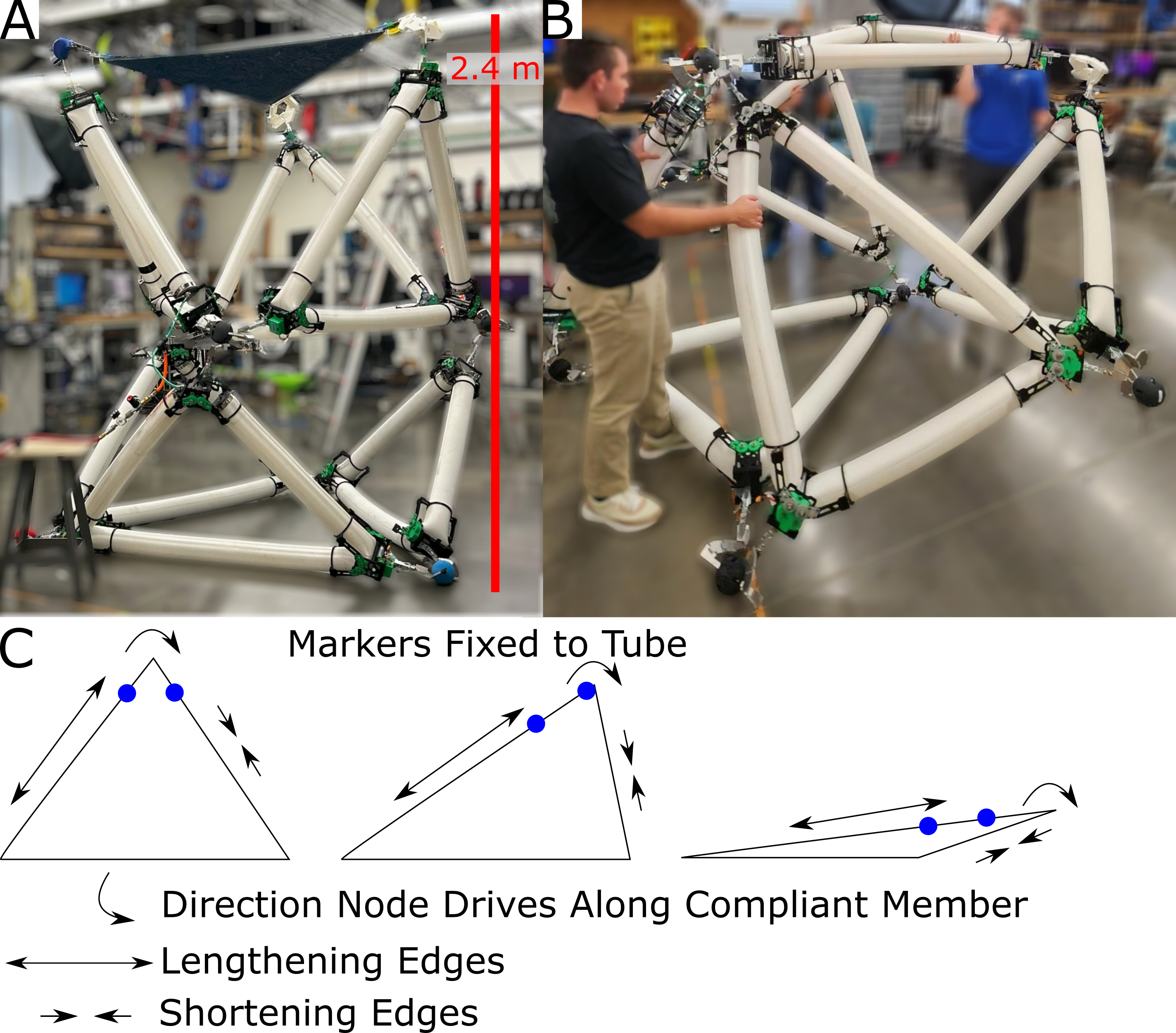} 
\caption{(A) An isoperimetric robot that is constructed from 7 triangles (6 inflated triangles, and a rigid triangle on top that serves as an interaction surface. (B) An isoperimetric robot in the shape of a hexagonal bipyramid, formed from 6 triangles. (C) Illustration of the operating principle of the isoperimetric robot.  Each triangle changes shape as the rollers move along the inflated tube, changing the position of the vertex, but maintaining a constant perimeter.}
\label{fig:isoperimetric}  
\vspace{-10 pt}
\end{figure}


This paper makes the following contributions: 

\begin{itemize}
\item We present a novel algorithm that determines if a graph can be partitioned into individual triangles, and then finds a valid partition. 
\item We enumerate enumerate all possible minimally rigid graphs up to 9 nodes (7 triangles) that can be partitioned into individual triangles, and all partitions of those graphs. 
\item We present a constructive method for combining partitioned graphs together to form larger structures that themselves are triangle-decomposable
\item We give an analysis of the reachable set of points by minimally rigid isoperimetric robots up to 9 nodes. 
\end{itemize}

\section{Related Work}

\subsection{Modular Robotics }
A robot composed of modular and reconfigurable elements would allow robotic systems to be able to adapt to a wide variety of tasks and environments \cite{yim2007modular_grandchallenge,seo2019modular}. Often the robots themselves are the primary building blocks of the structure \cite{rubenstein2012kilobot, romanishin2013_Mblocks}, while in other cases robots can form a shape from passive material (\cite{petersen2011termes}).   

The class of modular robots to which the Isoperimetric robots most closely belongs is truss robots, often referred to as variable geometry trusses. Truss robots consist of rigid links that can change their length connected at passive joints.  The tetrobot project used as its base unit a tetrahedron, and developed hardware and dynamic controllers for this class of robots \cite{hamlin1996tetrobot, hamlin1997tetrobot, lee2002tetrobot_dynamic, lee1999tetrobot_dynamics}. The tetrobot project enumerates a number of different shapes that the tetrahedral robots can take, including an algorithm for chaining tetrahedon elements together to create arbitrarily long chains. Hardware construction and controller design of these tetrahedral based structures has also been considered for lunar operations \cite{curtis2007tetrahedral_for_space, abrahantes2007gait_12tet}. There have been several other hardware constructions of these truss style robots \cite{lyder2008_odin, zagal2012deformable_Oct}. In \cite{spinos2021topological_VTT}, the authors determine which shapes the truss robot can take while it reconfigure itself without user intervention. Truss robots are also closely related to tensegrity robots \cite{shah2022tensegrity}, which have been proposed for space exploration \cite{bruce2014superball, sabelhaus2015sys_designSUPERball}, climbing through ducts \cite{friesen2014ductt}, and other tasks.  Determining how to combine tensegrity elements into large numbers of shapes has also been a topic of research \cite{aloui2018ellular_tensegrity_generation, wang2020tenseg_topology_design}.  In the case of the isoperimetric robot, the fundamental element is a triangle, which poses the problem of what shapes can be formed from triangles. While other similar other works that have described and enumerated tensegrity, tetrahedral, and truss configurations, this work describes and enumerates isoperimetric robots.

\subsection{Graph Decomposition}
This work also draws on prior contributions in graph theory and discrete optimization. The problem of partitioning a graph into triangles is a special case of an edge partition problem, which has been shown to be NP-complete \cite{holyer1981np_edge_part}. While extensive research has focused on characterizing rigid graphs, a full combinatorial characterization of rigidity in three dimensions in an open problem. In \cite{bartzos2018maximal}, the authors examine the number of embedding of minimally rigid graphs, and in doing so, enumerate all minimally rigid graphs up to 8 nodes. In constructing the graphs, the authors note that two operations, dubbed H1 and H2 steps and shown in Fig.~\ref{fig:Hstep}, are sufficient to construct all Geiringer graph with $\leq$ 12 vertices.  

\begin{figure}[tb]
\centering
\includegraphics[width=.4\columnwidth]{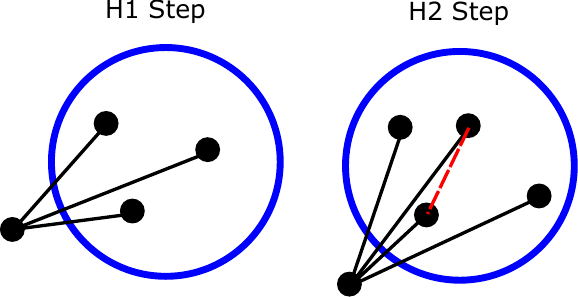} 
\caption{Two different Hennenberg steps utilized to construct graphs. In an H1 step, a new node is added with 3 connecting edges. For an H2 step, a new node with four new edges is added, and an edge between the connecting nodes is deleted. }
\label{fig:Hstep}  
\vspace{-10 pt}

\end{figure}

\section{Mathematical Definition of an Isoperimetric Robot}


We define the state of an isoperimetric robot using information about the connectivity of the edges (an undirected graph), the position of the vertices (the embedding) and how the edges are divided into triangles (the triangle partition).  We denote the graph  $\mathcal{G} = (\mathcal{V}, \mathcal{E})$ where $\mathcal{V}$ is the set of vertices or nodes and $\mathcal{E}$ is the set of edges. We define the embedded graph as a framework $(\mathcal{V}, \mathcal{E}, \bm p)$ where $(\mathcal{V}, \mathcal{E})$ is the graph and $\bm p \triangleq \bmx{\bm p_1^\intercal & \bm p_2^\intercal & \cdots & \bm p_n^\intercal}^\intercal$, $\bm p_i \in \R^3$ $\forall i$, is a set of points in Euclidean 3-dimensional space.

\begin{assume}
All graphs considered in this paper do not have self-loops and do not have duplicate edges. More precisely, \((i,j) \in \mathcal{E} \implies i \neq j\) and 
any (undirected) edge \((i,j) \in \mathcal{E}\) is unique.
\end{assume}

\subsubsection{Triangle Partition}
The overall goal is to partition the set of edges $E(\mathcal{G})$ into subsets, where each subset induces a subgraph isomorphic to $K_3$ (a triangle).
To do this, we define the notions of binary-valued \emph{triangle indicator matrices} and \emph{triangle partition matrices}. For this definition, \(M^j\) indicates the \(j\)th column of the matrix \(M\).

\begin{define}
    \label{def:triangleindicator}
    The matrix \(T \in \{0,1\}^{|E| \times N_T})\), \(N_T \in \Z_{\geq 0}\) is a triangle indicator matrix for a graph \(\mathcal{G} = (\mathcal{V}, \mathcal{E})\) if every column \(T^j\) is the indicator vector for the edge set of a subgraph of \(\mathcal{G}\) isomorphic to \(K_3\).
\end{define}

\begin{define}
    \label{def:trianglepartitionmatrix}
    The matrix \(T \in \{0,1\}^{|E| \times |E|/3}\) is a triangle partition matrix for a graph \(\mathcal{G} = (\mathcal{V},\mathcal{E})\) if all of the following conditions are satisfied:
    \begin{enumerate}
        \item \(T\) is a triangle indicator matrix, and
        \item \(T \vec{\bm 1} = \vec{\bm 1}\). Equivalently, the edge sets associated with each column of \(T\) form a partition of \(\mathcal{E}\).\footnote{The dimensions \(T \in \{0,1\}^{|E| \times |E|/3}\) for a triangle partition matrix follow from every edge being included in exactly one triangle (with three edges per triangle).}
    \end{enumerate}
\end{define}

Definitions \ref{def:triangleindicator}, \ref{def:trianglepartitionmatrix} imply that the entries of a triangle indicator or triangle partition matrix can be defined as 
\begin{align}
    T_{i}^j = \begin{cases}
    1 & \text{ if edge $i$ is in triangle $j$,} \\
    0 & \text{ otherwise.}
    \end{cases}
\end{align}
Each column $T^j$ of $T$ represents a different triangle and each row $T_i$ of $T$ represents an edge of the graph $\mathcal{G}$. 
Note that Definitions \ref{def:triangleindicator}, \ref{def:trianglepartitionmatrix} trivially imply that \(T^\intercal \vec{\bm 1} = 3\vec{\bm 1}\), meaning that each triangle contains exactly three edges.


The kinematics of an isoperimetric robot relating the changes in the edge lengths with motions of the nodes are described in detail in \cite{usevitch2020untethered}. For each graph, we can construct a rigidity matrix, $R(x) \in \R^{|\mathcal{E}| \times 3|\mathcal{V}|}$ can be used to determine whether or not the framework is infinitesimally rigid using the following result:

\begin{lemma} 
[\cite{connelly1993rigidity}]
   \label{lem:infinitesimallyrigid}
    The framework $(\mathcal{V}, \mathcal{E}, \bm p)$ is infinitesimally rigid in $\mathbb{R}^m$ if and only if its rigidity matrix $R(x)$ has rank $3n-6$, where $n$ is the number of nodes.
\end{lemma}





\section{Partitioning a Graph into distinct Triangles} \label{Sec:Partition}

We present three separate algorithms to partition a graph into distinct triangles: an exhaustive search algorithm, an integer programming routine with a pre-enumeration step, and an integer programming algorithm.

For an arbitrary graph to be decomposed into unique triangles, it must satisfy the following necessary (but not sufficient) conditions: 

\begin{itemize}
    \item The degree of each node in the graph must be even
    \item The total number of edges in the graph must be divisible by 3
\end{itemize}

\subsection{Exhaustive Search Routine}
If a graph meets the necessary conditions, a na\"ive approach to determining if a graph can be partitioned into triangles is the following:

\begin{itemize}
\item Identify all triangles in the graph
\item Select all possible combinations of triangles
\item Evaluate each combination of triangles to determine if each edge in the graph appears exactly once. 
\end{itemize}

The triangles triangles in a graph are precomputed using  algorithm \ref{alg_triangles}.  In this algorithm, we loop over each node in the graph, and look for nodes that have an intersection in their neighbor sets, where the neighbor set of node $i$ is denoted $N(i)$

\begin{algorithm} 
\caption{Triangle Pre-Enumeration}
\begin{algorithmic}
    \State $triangles \leftarrow \emptyset$
    \For{$i =0:n$}  \Comment{Loop over each node}
        \For{$v \in N(i)$}   \Comment{Loop over the nodes neighbors}
            \If{$v < i$} \Comment{Avoid duplicate checks}
                \For{$w \in N(i) \cap N(v)$} \Comment{Shared neighbors}
                    \If{$w > v$} \Comment{Avoid duplicate triangles}
                        \State $triangles \leftarrow triangles \cup \{ (i, v, w) \}$
                    \EndIf
                \EndFor
            \EndIf
        \EndFor
    \EndFor
    \State \Return $triangles$
\end{algorithmic}
\label{alg_triangles}
\end{algorithm}

The computational complexity of this algorithm, and the upper bound on the number of triangles is given in \cite{suri2011countingtriangles} as

\begin{equation}
    N_{Triangles}= \frac{1}{3}\sum_{i=1}^n \frac{d_i (d_i-1)}{2}
\end{equation}

where $d_i$ is the degree of node $i$. The number of possible combinations depends on the number of triangle in the graph. If there are $m$ candidate triangles, and $3k$ edges in the graph, the number of possible combinations to check is given by the binomial coefficient: 

\begin{equation}
N_{combinations}=\frac{n!}{(n-k)! k!}.
\end{equation}

Thus the number of triangles in a graph grows quadratically with the degree of the nodes of the graph, but the number of possible combinations grows with the factorial of the number of triangles in the graph. This procedure is tractable for small graphs, but exhaustively enumerating and searching the possible combinations rapidly exceeds our computational resources for large, dense graphs. For this reason, we pursue an integer programming approach.  

\subsection{Integer Programming Formulations for Determining Triangulated Rigidity}

We present two integer programming techniques for decomposing a graph into triangles. In the first, each triangle in the graph is identified as an initial step. In the second, the decomposition is solved end-to-end by an integer convex program without pre-enumeration.

\subsubsection{Triangle Pre-Enumeration ILP}

The triangle pre-enumeration technique consists of two steps. The first step precomputes all possible triangles within the input graph, using the procedure in Algorithm \ref{alg_triangles}. The second step uses integer linear programming solver to partition edges into a valid graph decomposition.

For the pre-enumeration technique, we encode all possible triangles with a triangle indicator matrix \(T\) as per Definition \ref{def:triangleindicator}.
Since the matrix \(T\) contains all possible triangles within the graph and is not guaranteed to be a triangle partition matrix, it follows that some triangles within \(T\) may have intersecting edge sets.

To obtain a triangle partition matrix from \(T\) as per Definition \ref{def:trianglepartitionmatrix}, all edges must be included in exactly one triangle. Mathematically this can be expressed as the constraint \(Tx^* = 1\), where the \(j\)th entry of the indicator vector \(x^* \in \{0,1\}^{N_T}\) is 1 if triangle \(j\) is included in the partition and 0 if it is excluded. The final partition matrix is then \(T^* \triangleq \bmx{T^{j_1} & T^{j_2} & \cdots & T^{j_m}}\) where \(x_{j_k} = 1\ \forall k=1,\ldots,m\).


Finding the indicator vector \(x^*\) can be posed as an integer linear programming (ILP) feasibility problem as follows:
\begin{alignat}{2}
\begin{aligned}
    x^* = &\underset{x \in \{0,1\}^{N_T}}{\arg\min}&\hspace{1em} 0 
    & &\text{s.t.}& \hspace{1em}Tx = \bm 1 \label{eq:preenumeratedILP}
\end{aligned}
\end{alignat}
Any feasible point to this ILP is a valid indicator vector specifying which columns of \(T\) form a triangle partition matrix.

The total runtime of the Triangle Pre-Enumeration ILP method is the runtime of Algorithm \ref{alg_triangles} added with the runtime of the ILP in \eqref{eq:preenumeratedILP}. The ILP in \eqref{eq:preenumeratedILP} can be solved by any standard optimization solver such as Gurobi \cite{gurobi}, MOSEK \cite{mosek}, HiGHS \cite{huangfu2018parallelizing}, or SCIP \cite{BolusaniEtal2024OO}.

\subsubsection{End-to-End IQCQP}

The second integer programming technique involves computing a triangle partition matrix directly from the graph incidence matrix without a pre-enumeration step.
Let $d= \bmx{d_1 & \cdots & d_{|E|}}$ be the vector of degrees of the line graph $\mathcal{L}(\mathcal{G})$.
Let $D$ denote the incidence matrix of $\mathcal{G}$, and let $|D|$ denote the entrywise absolute value of the matrix $D$.
Let $\mathcal{L}(\mathcal{G})$ denote the line graph of $\mathcal{G}$, and let $L$ be the Laplacian matrix of $\mathcal{L}(\mathcal{G})$.
We will show in this section that this problem can be expressed as the following integer quadratically-constrained quadratic program (IQCQP):
\begin{alignat}{3}
&\begin{aligned}
    T^* = \underset{T \in \{0,1\}^{|E|\times |E|/3}}{\arg\min}\hspace{1em} &&& 0
\end{aligned} \nonumber \\
&\hspace{3em}\begin{aligned}
    \text{s.t.} &&& T^\intercal x = 3 \vec{\bm 1} \\
    &&&T x = \vec{\bm 1} \\
    &&&(T^j)^\intercal L (T^j) - d^\intercal T^j \leq -6\ \forall j=1,\ldots,N_T^*, \\
    &&& |D| T^j \leq 2\ \forall j=1,\ldots,N_T^* \label{eq:IQCQP}
\end{aligned}
\end{alignat}
The optimization variables of this IQCQP are the entries of the matrix \(T \in \{0,1\}^{|E| \times |E|/3}\).
The first constraint of \eqref{eq:IQCQP} ensures the columns of \(T\) have at least 3 non-zero entries, and the second constraint enforces each edge to be present in exactly 1 column.
The third and fourth sets of constraints enforce each column of \(T^*\) to represent an indicator vector for an edge set of a subgraph of \(\mathcal{G}\) isomorphic to \(K_3\).
The third set of constraints enforces that all edges in each triangle are adjacent to each other.
This is captured in the following inequality:
\begin{align}
    \label{eq:inequality}
    (T^j)^\intercal L (T^j) - d^\intercal T^j \leq -6,\ \forall j \in \{1,\ldots, N_T^*\}.
\end{align}
On its own, the condition in \eqref{eq:inequality} is necessary but not sufficient to ensure that a given triangle \(T^j\) is isomorphic to \(K_3\).\footnote{As a counterexample, consider a 4-node graph with edges \((1,2),\ (1,3),\ (1,4)\).}
The second set of conditions enforces that no single node is part of all three edges:
\begin{align}
    \label{eq:incidence_inequality}
    |D| T^j \leq 2 \vec{\bm 1}.
\end{align}

\begin{lemma}
    An induced triangle subgraph \(T^j\) with edges $j_1,j_2,j_3 \in E(\mathcal{G})$ is isomorphic to $K_3$ if and only if inequalities \eqref{eq:inequality} and \eqref{eq:incidence_inequality} are simultaneously satisfied.
\end{lemma}
\begin{proof}
\textbf{Sufficiency:} Suppose \(T^j\) is the indicator vector for the edge set of a triangle isomorphic to \(K_3\).
\(T^j\) can be written as $T^j = I^{j_1} + I^{j_2} + I^{j_3}$ where $I^j$ represents the $j$th column of the identity matrix. Observe that
    \begin{align}
        (I^{i})^\intercal L I^{k} = L_i^k = \begin{cases}
            d_i & \text{ if $i = k$}, \\ 
            -1 & \text{ if $i \neq k$ and edge $i$ is adjacent} \\ & \text{ to edge $k$}, \\
            0 & \text{ otherwise.}
        \end{cases}
        \label{eq:ILIcases}
    \end{align}
    Also observe that $d^\intercal T^j = d^\intercal (I^{j_1} + I^{j_2} + I^{j_3}) = d_{j_1} + d_{j_2} + d_{j_3}$. From the symmetry of $L$ (recall that we are considering undirected graphs) it follows that
    \begin{align}
        \label{eq:k3sum}
        (T^j)^\intercal L (T^j) - d^\intercal T^j = 2(L_{j_1}^{j_2} + L_{j_1}^{j_3} + L_{j_2}^{j_3}).
    \end{align}
    Since \(T^j\) is the indicator vector for the edge set of a triangle isomorphic to \(K_3\), all three edges \(j_1, j_2, j_3\) are adjacent to each other and the RHS of \eqref{eq:k3sum} is equal to -6, satisfying constraint \eqref{eq:inequality}.

    Next, since the triangle represented by \(T^j\) is isomorphic to \(K_3\), there exist three nodes \(n_1, n_2, n_3\) such that $|D|I^{j_1} = I^{n_1} + I^{n_2}$, $|D|I^{j_2} = I^{n_2} + I^{n_3}$, and $|D|I^{j_3} = I^{n_3} + I^{n_1}$.
    This follows from the definition of the incidence matrix \(D\). We therefore have \(|D|T^j = |D|(I^{j_1} + I^{j_2} + I^{j_3}) = 2(I^{n_1} + I^{n_2} + I^{n_3})\), which implies that \(|D|T^j \leq 2 \vec{\bm 1}\) for all \(j\). Constraint \eqref{eq:incidence_inequality} is therefore satisfied.

\textbf{Necessity:} We prove the contrapositive.
Suppose that \(T^j\) is the indicator vector for the edge set of a subgraph that is \emph{not} isomorphic to \(K_3\). The objective is to show that either constraint \eqref{eq:inequality} or \eqref{eq:incidence_inequality} is not satisfied. Since the subgraph represented by \(T^j\) contains three edges due to the first constraint in \eqref{eq:IQCQP}, the subgraph not being isomorphic to \(K_3\) implies that the number of nodes within the subgraph is either 4, 5, or 6.\footnote{This follows because a subgraph with three nodes and three edges must trivially be isomorphic to \(K_3\).}

\emph{Case 1:} Suppose there exists a node \(n^*\) in the subgraph represented by \(T^j\) that belongs to all edges. Since there are only 3 edges, this implies that the number of nodes in the subgraph is 4. An example is given by the subgraph with edges (1,2), (1,3), (1,4) with \(n^* = 1\). These graphs are clearly non-isomorphic to \(K_3\) since they must contain four nodes.
It follows that for the three edges \(j_1,j_2,j_3\) in \(T^j\) we have $|D|I^{j_1} = I^{n_1} + I^{n_2}$, $|D|I^{j_2} = I^{n_1} + I^{n_3}$, and $|D|I^{j_3} = I^{n_1} + I^{n_4}$.
Therefore \(|D|T^j = 3I^{n_1} + I^{n_2} + I^{n_3} + I^{n_4}\), which implies that \(\max \pth{|D|T^j} = 3\) and therefore constraint \eqref{eq:incidence_inequality} is not satisfied.

\emph{Case 2:} Suppose that \(T^j\) represents a subgraph not isomorphic to \(K_3\) and \(\max \pth{|D|T^j} < 3\). Since constraint \eqref{eq:incidence_inequality} is feasible, we demonstrate that constraint \eqref{eq:inequality} is not satisfied.
Observe that, for this case, there must exist two non-adjacent edges \(j_1, j_2\) in the subgraph. This can be easily verified by noting that the only possible subgraphs where all three edges are adjacent are \(K_3\) or a graph matching the conditions of Case 1; i.e., \(\max \pth{|D|T^j} = 3\). Since there exist two non-adjacent edges, by \eqref{eq:ILIcases} we have \((I^{j_1})^\intercal L I^{j_2} = L_{j_1}^{j_2} = 0\). By \eqref{eq:k3sum} we therefore have 
\begin{align}
    \begin{aligned}
        (T^j)^\intercal L (T^j) - d^\intercal T^j &= 2(L_{j_1}^{j_2} + L_{j_1}^{j_3} + L_{j_2}^{j_3}) \\
        &= 2(0 + L_{j_1}^{j_3} + L_{j_2}^{j_3}) > -6.
    \end{aligned}
\end{align}
Therefore constraint \eqref{eq:inequality} is not satisfied, which concludes the proof.
\end{proof}

\section{Composition Algorithms}

\subsection{Enumeration from minimally rigid graphs}
In this section we address the problem of enumeration:  what are all of the minimally rigid graphs, up to a given number of nodes, that can be decomposed into triangles? 

In \cite{bartzos2018maximal}, the authors provide an enumeration of all minimally rigid graphs based on Hennenberg steps up to 8 nodes. Extending this enumeration to an exhaustive enumeration of graphs with 9 nodes requires a prohibitive amount of computational resources. However, the only graphs that can have valid partitions are those that have all nodes with even degree. If the final step is an H1 step, the last node will always have had  odd degree, and cannot be partitioned.  The only viable partitions for 9 nodes graphs will be encountered if the degree of each node is even.  Therefore, we extend their enumeration to 9 nodes by taking all known graphs with 8 nodes, and evaluating all 9 node graph that result from taking an H2 step.

Using the enumeration of minimally rigid graphs, we then apply the exhaustive search algorithm from the previous section to identify which are decomposable. The total numbers of the candidate graphs are given in table \ref{tab:Enumeration} for each number of nodes. The graphs for which a successful partition are found are shown in Fig.~\ref{fig:enumerate}.  We also list the workspace of one node of each of these graphs, which we will describe in \ref{Sec:Workspace}. In order to describe these graphs, we utilize a convention used by other researchers to describe arbitrary graphs. We take the upper triangular portion of the adjacency matrix, interpret that sting of 1's and 0's as a binary number, and label the graph with the number that corresponds to the value of the binary number converted to base 10.  Thus the number assigned to each graph uniquely describes the graph's adjacency matrix. 

Interestingly, there are some graphs for which multiple partitions of the same graph are possible. Graph 26622, an octahedron, has two possible partitions, one with a labeled triangle on the top, and another with a labeled triangle as the base. There are three possible partitions of two stacked octahedrons (graph 60243677150). While these graphs are identical, the motion of the resulting robot is different due to the different constraints resulting from the triangle partitions. This is evidenced by different sizes of their workspaces. 

\begin{table}

\centering
\begin{tabular}{||c c c c||}
 \hline
 \# of Nodes & MR Graphs & Satisfy NC & Partitions \\ [0.5ex] 
 \hline\hline
 6 & 4 & 1 & 2 \\ 
 \hline
 7 & 26 & 2 & 1 \\
 \hline
 8 & 374 & 6 &  3\\
 \hline
 9 & - & 60 &  13\\ [1ex] 
 \hline

\end{tabular}

\label{tab:Enumeration}
\caption{\textup{The number of minimally rigid graphs for different number of nodes, how many of those graphs satisfy the necessary conditions, and the total number of partitions (note that some graphs that satisfy the necessary conditions can be partitioned in multiple ways).}}
\vspace{-10 mm}
\end{table}

\begin{figure*}[tb!]
\centering
\includegraphics[width=1.6\columnwidth]{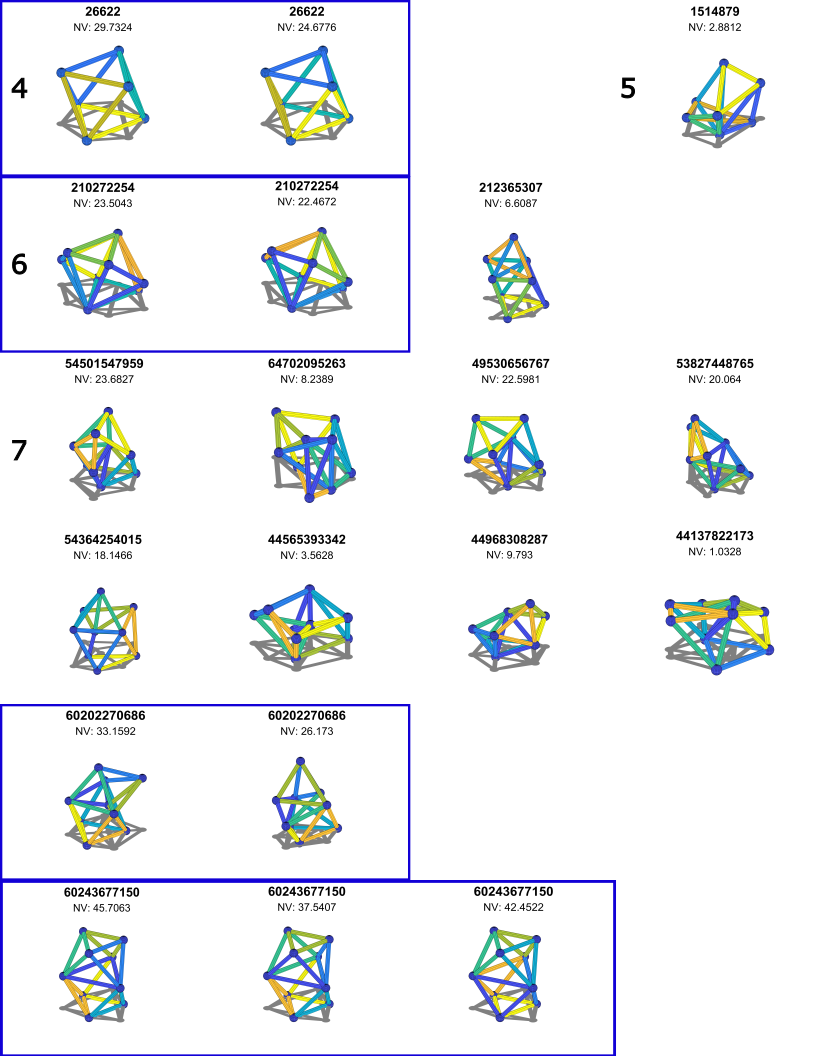} 
\caption{All partitions of all minimally rigid graphs with 9 or fewer nodes that can be partitioned into triangles. Note that some graphs have multiple possible partitions. Graphs are grouped according to the number of component triangles.  Graphs that are identical but have different partitions are grouped in blue boxes.  For each graph, the normalized workspace volume NV of the top node is given.  }
\label{fig:enumerate}  
\vspace{-4mm}
\end{figure*}

\paragraph{Embedding the Graphs} Once the graphs are constructed, we must then find an embedding (assign values to the nodes coordinates of the graph) so that we can examine them as robotic systems. In this study we find embeddings using a multidimensional scaling \cite{cox2000multidimensional}, as implemented by the Matlab function ``mdscale."  

One challenge is that the embedding that result from the multidimensional scaling for some of these graphs occur at configurations where the graph is not infinitesimally rigid. Mathematically, the robot is not rigid when the rigidity matrix, $R$, loses row rank. We quantify the rigidity of the graphs with the worst case rigidity index, $\lambda_{WCR}$ formulated in \cite{zelazo2012rigidity,zelazo2015decentralized}:  

\begin{equation}
\label{eq:WCR}
  \lambda_{WCR}=\frac{\lambda_7}{\sum_{i=1}^{3n} \lambda_i}=\frac{\lambda_7}{ \text{tr}(R(x)^TR(x))}=\frac{\lambda_7}{\sum_{i=1}^{N_L} (L(x)_i)^2} 
\end{equation}

where $\lambda_7$ is the seventh smallest eigenvalue of the matrix $R(x)^TR(x)$. In order to find embeddings that are infinitesimally rigid, we use the algorithm in \ref{alg_embed}. With this approach, we generate a large number of embedding of the graph with some induced randomness, and then select the one that has the largest worst-case rigidity index.   

\begin{algorithm} \label{alg_embed}
\caption{Determine a Embedding}
\begin{algorithmic}
    \State $WCR \leftarrow 0$
    \For{$i =0:n_{trials}$}  \Comment{Number of Embeddings to Try}
        \State $D=0$ \Comment{Initialize Distance Matrix}
        \For{$i \in 1:n$}   
               \For{$j \in [j:n]$} 
                   \If{$A(i,j) == 1$}  \Comment{ Connected nodes}
                       \State $D(i,j)=rand()$ 
                   \Else 
                       \State $D(i,j)=10$ \Comment{Disjoint nodes}
                   \EndIf
             \EndFor
        \EndFor
        \State x=mdscale(D) \Comment {Compute embedding}
        \State WCR=wcr(x) \Comment{Eq. \ref{eq:WCR}}
        \If{$WCR>WCR_{max}$}
            \State $x_{max}=x$
        \EndIf
    \EndFor
    \State \Return $x_{max}$
\end{algorithmic}
\end{algorithm}

\subsection{Constructive Methods for Combining Partitioned Graphs}

We present a method for combining two infinitesimally rigid graphs that are partitioned into edge-disjoint triangles, into a single graph that is also infinitesimally rigid and decomposes into unique triangles.  Conceptually, this involves deleting a labeled triangle from one graph, and merging the 3 nodes of that triangle with three nodes of the other graph 


Let  $\mathcal{G}_1=(\mathcal{U}_1, \mathcal{E}_1)$, with node locations $ x_1= (p_1, ... ,p_{n_1})$ and $\mathcal{G}_2(\mathcal{V}_1, \mathcal{E}_2)$ with node positions $x_2= (q_1 ... q_{n_2})$ be infinitesimally rigid frameworks in \(\R^3\). Let there be a triangle (a subgraph isomorphic to $K_3$) in $(\mathcal{G}_1$ composed of nodes $u_1, u_2, u_3$, and edges $e_{12}, e_{23}, e_{13}$. Let $\mathcal{G}_2$ have 3 nodes positioned at the same relative distances as the triangle in $\mathcal{G}_1$, meaning that a homogeneous transform $M$ exists such that $[p_{1}, p_{2}, p_3]=M[q_{1}, q_{2}, q_{3}]^T$. We define the graph $\mathcal{G}_{1'}=\mathcal{G}_1-\{e_{12}, e_{23}, e_{13} \}$, meaning that the graph $\mathcal{G}_{1'}$ is the graph resulting from deleted the connecting edges of the triangle.

\begin{lemma}
     Let $\mathcal{G}_{combine}$ be the graph formed by joining graphs \(\mathcal{G}_1\) and \(\mathcal{G}_2\), i.e., $\mathcal{G}_{combine} = \mathcal{G}_{1'} \cup \mathcal{G}_{2}  \text{ with } \{v_1 \sim u_1, v_2 \sim u_2, v_3 \sim u_3\}$, where $\sim$ denotes combining the nodes of the two graphs.
     Let $x_c=[p_1...p_{n_1},M q_4...Mq_{n_2}]$ be the node positions for \(\mathcal{G}_{combine}\).
     Then the framework $(\mathcal{G}_{combine}, x_c)$ is infinitesimally rigid. 
\end{lemma}

\begin{proof}
  The rigidity matrix of ($\mathcal{G}_1$, $x_1$) is 
    \begin{equation}
     R_1 = \begin{bmatrix}R_C & 0 \\ R_{1C} &  R_{11}\end{bmatrix}.
\end{equation}
where $R_C$ is the rigidity matrix of the deleted triangle. Let $R_2$ be the rigidity matrix of $(\mathcal{G}_2, x_2)$  The rigidity matrix of $(\mathcal{G}_{combine}, x_c)$ is

\begin{equation}
    R_{combine} = \begin{bmatrix}R_2 &  0  \\ R_{1C} &  R_1 \end{bmatrix}
\end{equation}

The deleted edges corresponding to the triangle $R_c$  are linearly dependent with the rows of $R_2$, because the framework $(\mathcal{G}_2, x_2)$ is infinitesimally rigid.  Thus the rank of $R_{combine}$ is $(3n_1-6 + 3n_2-6 -3)$. As $n_C = n_1+n_2-3$, the rank of $R_{combine}$ is $3n_C-6$. Thus the framework $(\mathcal{G}_{combined}, x_c)$   is infinitesimally rigid by Lemma \ref{lem:infinitesimallyrigid}.
\end{proof}

If the triangle at which the two graphs are joined is part of the partition in each subgraph, then the overall graph can also be partitioned

\begin{figure}[htb]
\centering
\includegraphics[width=.5\textwidth]{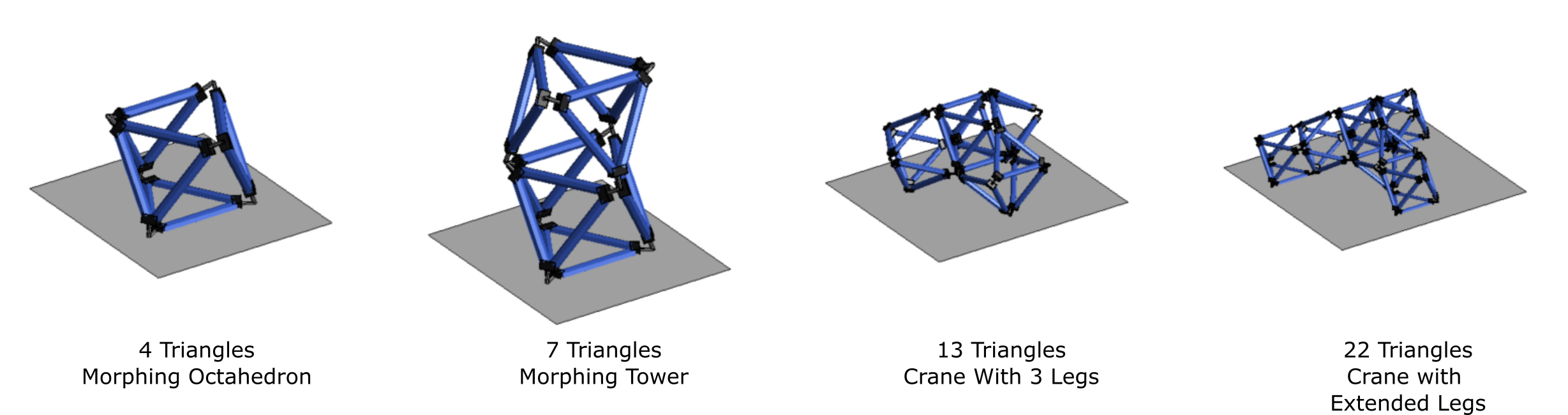} 
\caption{Shapes formed by combining octahedral units into chains and branching shapes }
\label{fig:construct_shapes} \vspace{-4 mm}
\end{figure}

Using this constructive algorithm, we can combine any of the graphs presented in Fig.\ref{fig:enumerate}. We have constructed several different graphs using this constructive algorithm, with the results shown in Fig.~\ref{fig:construct_shapes}.

\section{Results: Computational Timing}

We evaluate the computational performance of the three partition algorithms presented in Section \ref{Sec:Partition}. We generate a test set of random graphs by starting with an octahedron, identifying three random nodes in the graph, and merging the base graph with a new octahedron using the procedure in Lemma 3. Denoting the number of nodes in the graph as $|\mathcal{V}|$, 20 random graphs were generated for each value of $|\mathcal{V}|$ in $[9, 15, 21, \ldots, 297]$. For each graph, the Exhaustive Search Routine, Triangle Pre-Enumeration ILP, and End-to-End IQCQP algorithms were run and the runtimes were recorded. Due to computational limitations and the NP-hardness of integer programming, a 100-second time limit was imposed on the runtimes. Integer programs were solved using the Gurobi optimizer \cite{gurobi}. All experiments were conducted on an AMD Ryzen Threadripper Pro 5975wx workstation processor.

A plot with the results is shown in Figure \ref{fig:runtimes}. The exhaustive search reaches the memory limit beginning with graphs of 18 nodes, while the end-to-end ICQCQP reaches the time limit beginning at 24 nodes. However, the algorithm with pre-enumeration enables computation up to 297 nodes in approximately 1 second, indicating that this method is tractable for large graphs. 

\begin{figure}
    \centering
    \includegraphics[width=\columnwidth]{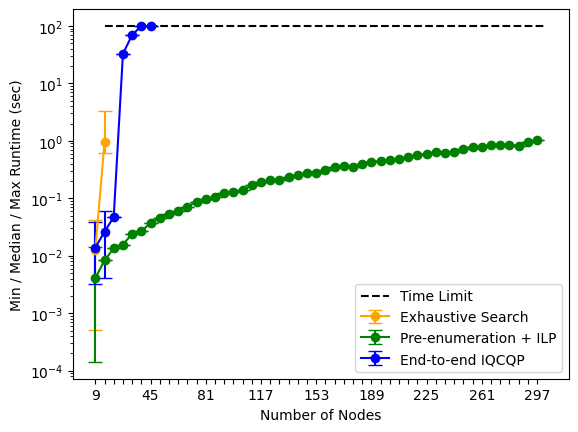}
    \caption{Plot comparing runtimes for exhaustive search, pre-enumeration ILP, and end-to-end IQCQP. Center dots show the median runtime, while upper and lower whiskers show min / max runtimes respectively. All experiments ran with a time limit of $100$ seconds. Only two entries for exhaustive search are shown since the algorithm encountered out-of-memory errors after that point. End-to-end IQCQP runtimes beyond 27 nodes are omitted due to runs hitting the timeout limit.}
    \label{fig:runtimes}
    \vspace{-2mm}
\end{figure}

\section{Results: Workspace Characterization} \label{Sec:Workspace}
\begin{figure}[tb]
\centering
\includegraphics[width=.5\textwidth]{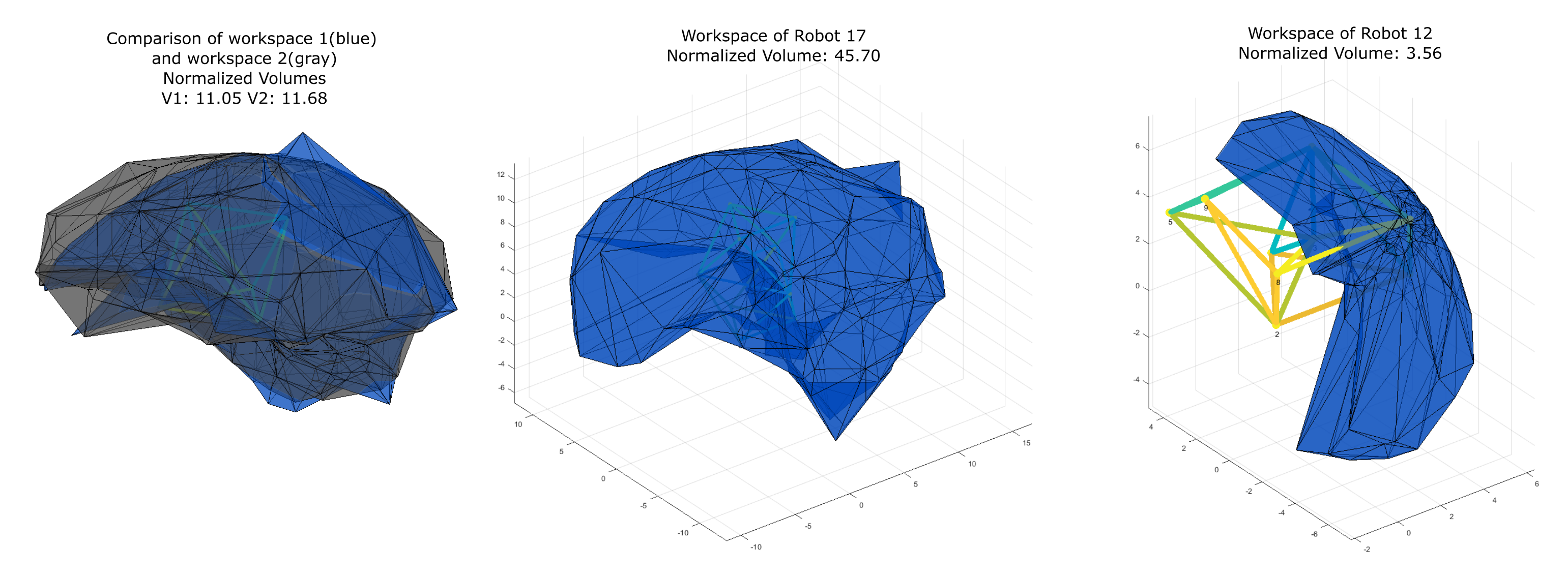} 
\caption{Visualization of three robot workspaces: (1) Overlap of two robots with identical graphs, (2) largest achievable workspace, and (3) smallest constrained workspace.}
\label{fig:workspace_vis}  
\vspace{-2mm}
\end{figure}

We now characterization the motion of the new robot shapes given in Fig.~\ref{fig:enumerate}. We provide a quantitative comparison by determining the reachable workspace of a single vertex of each robot, while three other ``base nodes" are constrained to be stationary. The support nodes we selected are shown as the base nodes in Fig.~\ref{fig:enumerate}.

Determining the workspace of each robot requires solving the inverse kinematics, meaning that we specify the motion of the controlled nodes of the graph in operational space, and determine how the robot must actuate to achieve this motion. We phrase the inverse kinematics as the following optimization problem:

\begin{align}  
   &\min_{\dot{x}_i}    
   \begin{aligned}[t]
        \| R(x) \dot{x}\|^2  \label{Motion_Primitive}
   \end{aligned} \\
   &\text{subject to} \notag \\
   &\quad \begin{bmatrix} A_i \\ C_i \\ 
 T^T R(x) \end{bmatrix} \dot{x}_i=\begin{bmatrix} b_i \\ 0 \\ 0 \end{bmatrix}  \label{eq:lin_control} \\
 &\quad D\dot{x}_i \geq 0 \label{eq:constraint}
\end{align}
where $\dot{x}_i$ is the velocity of all nodes of the robot at time $i$. 
In this case, $C \dot{x} = 0$ constrains the base nodes of the robot to be stationary, $A \dot{x} = b$ constrains one of the nodes to move in a specified target direction. $T^T R(x) \dot{x}=0$ is the constraint that the perimeter of each triangle remains constant, where $T$ is the triangle indicator matrix of the graph.

We also include the inequality constraint $D\dot{x}_i \geq 0$. We move the robot while enforcing the following constraints:
\begin{itemize}
\item Individual edge lengths must remain above a minimum length of 0.2
\item The robot's worse case rigidity index must remain above a threshold of 0.005, as defined in eq.~\ref{eq:WCR}.
\end{itemize}
After each time step we evaluate these constraints. If violated, we return to the previous time step and enforce the gradient of the constraint as a linear constraint of the form $D\dot{x}_i \geq 0$, using the same technique used in \cite{usevitch2020locomotion}.  

To calculate the workspace, we generated a set of 200 points on the surface of a sphere of radius 6, significantly larger than the robot can reach. We then compute the sphere's triangulation. Each robot was positioned within the sphere and commanded to move its designated movement node in a straight line to each target point on the sphere’s surface. For each attempt, the final position reached by the movement node was recorded until the robot could no longer continue due to the constraints. These final positions were used to compute the volume of the triangulation, to give an approximate volume of the workspace. This procedure is visualized in the video attachment.

To allow comparisons between robots with different numbers and lengths of edges, the workspace volume was normalized by dividing it by the cube of the longest edge length of each robot. This normalization controls for differences in scale and initial edge length, enabling a fair comparison of workspace across robots with differing geometries and triangle counts.

Robots with identical graphs exhibit variability in their workspaces due to differences in how they are partitioned into triangles. This can be seen in the boxed groupings in Figure \ref{fig:enumerate}. The most drastic case is that of graph 60243677150 where the workspace ranges from 37.54 to 45.7. Similar differences occur in graph 26622. Robots with edges that span through the middle of the graph tend to have limited ranges of motion. These graphs in their initial configuration are closer to singularities. Graph 26622 has an initial WCRI of 1.67, which is the highest of all the graphs we explored. Graph 44565393342 had the smallest initial WCRI of 0.01. The initial WCRI cannot predict the volume of the workspace. Robots based on the same graph have the same WCRI regardless of orientation or configuration, yet have different workspace volumes. Though 60243677150 had the largest workspace, it has a comparably small initial WCRI of 0.47. 

This study has allowed for comparison of the different workspaces of these robots. However, the movement strategy is not necessarily optimal, as there may be other movements that avoid singular configurations and can reach points outside of the current workspace. In addition, future work could examine different embeddings (node positions of the graphs) that may enable improved behavior.

\section{Conclusion}
In this paper we have expanded the number of robots that can be built from isoperimetric triangles. This allows for an increased number of robot shapes and types. These robots may form a viable candidate for space exploration, as they can stow in small volume when deflated, and then inflate to form large structures that can support substantial loads.  In future work, new hardware development could lead to nodes for the isoperimetric robot that allow for the compliant members to bend along multiple axes. This would remove the constraint that each tube must remain a planar triangle, and replace it with the constraint that the graph must be decomposed into Euler paths.  This would increase the number of candidate robot configurations.

\bibliographystyle{IEEEtran}
\bibliography{IEEEabrv,bibliography}

\end{document}